%% file: main.tex
\documentclass[smallabstract,smallcaptions]{dccpaper}

\usepackage{amsmath}
\usepackage{amssymb}
\usepackage{graphicx}
\usepackage{booktabs}
\usepackage{url}
\usepackage{tikz}
\usetikzlibrary{positioning,arrows.meta}
\definecolor{ourblue}{HTML}{2A78D6}
\usepackage{amsthm}
\newtheorem{proposition}{Proposition}

\newcommand{\ours}{PLEDGE}
\newcommand{\taupx}{\tau_{\mathrm{px}}}
\newcommand{\taush}{\tau_{\mathrm{sh}}}
\newcommand{\tauC}{\tau_{C}}
\newcommand{\taual}{\tau_{\alpha}}

\makeatletter
\renewcommand\section{\@startsection{section}{1}{\z@}{-8pt \@plus -2pt \@minus -1pt}{3pt}{\normalfont\normalsize\bfseries}}
\renewcommand\subsection{\@startsection{subsection}{2}{\z@}{-7pt \@plus -2pt \@minus -1pt}{2.5pt}{\normalfont\normalsize\itshape}}
\def\Section#1{\section*{#1}}
\def\SubSection#1{\subsection*{#1}}
\makeatother

\begin{document}

\title{\large\textbf{\ours: Progressive Lossy Error-Bounded Compression\\of 3D Gaussian Splatting}}

\author{Mingze Xia\thinspace\textsuperscript{*}, Yuxiao Li\thinspace\textsuperscript{\dag},
Baixi Sun\thinspace\textsuperscript{\ddag}, Wenbo Li\thinspace\textsuperscript{*},
Pu Jiao\thinspace\textsuperscript{\S},\\
Sheng Di\thinspace\textsuperscript{\ddag}, Franck Cappello\thinspace\textsuperscript{\ddag},
Hanqi Guo\thinspace\textsuperscript{\dag}, Xin Liang\thinspace\textsuperscript{*}\\[0.3em]
{\footnotesize\textsuperscript{*}Oregon State University\quad
 \textsuperscript{\dag}The Ohio State University\quad
 \textsuperscript{\ddag}Argonne National Laboratory}\\
{\footnotesize\textsuperscript{\S}The University of Texas Rio Grande Valley}\\[0.15em]
{\footnotesize\texttt{\char123 xiami,liwenb,lianxin\char125 @oregonstate.edu}\quad
 \texttt{\char123 li.14025,guo.2154\char125 @osu.edu}}\\
{\footnotesize\texttt{\char123 bsun,sdi1,cappello\char125 @anl.gov}\quad
 \texttt{pu.jiao@utrgv.edu}}
}

{\setlength{\tabcolsep}{0pt}\maketitle}
\thispagestyle{empty}

\begin{abstract}
Existing 3D Gaussian Splatting (3DGS) codecs report reconstruction
quality after compression, and give users little control over the
distortion of individual primitives. We present PLEDGE, a
training-free, progressive, error-bounded codec that translates four
user-declared tolerances into per-primitive quantization bounds:
screen-space position, footprint deformation due to rotation and
scale, view-dependent color, and opacity. The bounds apply throughout
a declared camera domain, including previously unseen viewpoints.
Position and opacity are bounded directly; footprint deformation is
bounded under the rasterizer's affine projection model, and color is
bounded over all viewing directions. PLEDGE organizes quantized residuals into a
progressive bitstream with explicit error bounds for every valid
prefix. Experiments across thirteen scenes verify all declared
per-mechanism bounds and demonstrate competitive rate--distortion
performance. The four tolerances incur substantially different rate
costs, so users can spend rate where their application needs it.
\end{abstract}

\Section{1. Introduction}

A trained 3D Gaussian Splatting (3DGS) model~\cite{kerbl3dgs} consists of
anisotropic Gaussian primitives, each represented by a position, a covariance
factored into scale and rotation, an opacity, and spherical-harmonic (SH)
color coefficients. A benchmark scene typically occupies several hundred
megabytes, and the codecs built to shrink it~\cite{bagdasarian3dgszip} optimize
and report aggregate reconstruction quality.

Attribute-level control matters when a scene is consumed for more than
viewing: an occlusion mask for AR needs primitives to stay within a pixel of
where they were trained, a measurement or navigation task needs geometry more
than color, and a streaming preview needs the opposite. The rate costs of these attributes are also
asymmetric: at a tight color tolerance, degree 1--3 SH coefficients
account for \(48\%\) of our bitstream, compared with \(14\%\) for position.
Error-bounded scientific compression~\cite{di2016sz} controls numerical
distortion, but applying it to 3DGS requires
connecting parameter errors to image-space effects. A one-pixel footprint
displacement is directly interpretable, whereas a quaternion error of
\(10^{-3}\) is not. Image-space distortion depends on camera
parameters and must account for viewpoints beyond those available during
encoding.

Clients should also be able to reconstruct a scene at different download
budgets with predictable error bounds. Existing progressive codecs~\cite{chen2026pcgs} offer no such guarantee for
intermediate reconstructions.

\textbf{Contributions.} \ours\ (Progressive Lossy Error-bounded
Declared-tolerance Gaussian-Splatting Encoder) needs only a trained 3DGS model
and its camera calibration. It inverts the four tolerances into per-primitive
quantization bounds over a declared camera domain
(Figure~\ref{fig:framework}), then codes the residuals with a hierarchical
predictor and progressive bit planes, so every valid prefix of the bitstream
carries its own bound (Proposition~\ref{prop:ladder}).

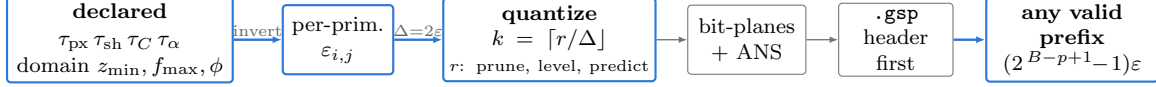
\begin{figure}[tb]
\centering
\input{figs/framework}
\vspace{-2em}
\caption{\label{fig:framework}\ours\ overview. Four user-declared tolerances and a camera domain
give per-primitive quantization steps \(\Delta=2\varepsilon\); by
Proposition~\ref{prop:contract}, everything after quantization affects rate
but not the guarantee.}
\end{figure}

\Section{2. Related work}
3DGS codecs split by whether they need the training images.
\emph{Retraining} methods --- vector-quantized or entropy-regularized
representations~\cite{niedermayrcompressed3dgs}, anchor-based
models~\cite{luscaffoldgs}, their context models~\cite{chenhac} and the
progressive member of that family~\cite{chen2026pcgs} --- reach the best rates
but need the dataset. \emph{Training-free} methods post-process a trained
model: LightGaussian~\cite{fanlightgaussian} prunes, Self-Organizing
Gaussians~\cite{morgensternsog} sort attributes into 2D grids,
SPZ~\cite{nianticspz} is a compact format,
FlexGaussian~\cite{tianflexgaussian} quantizes to a quality target, and
geometry can go to a point-cloud coder~\cite{mpegtmc13}. \ours\ is
training-free; to our knowledge none bounds error. FlexGaussian comes closest
with a \(1\,\mathrm{dB}\) target, but a scene average says nothing about an
individual primitive or an unseen viewpoint.

\ours\ adapts the error-bounded discipline of scientific compression, whose
bounds we must map from image to parameter space, to embedded
coding~\cite{shapiro1993ezw,taubman2002jpeg2000}, and bounds every valid stream
prefix rather than only the complete file.

\Section{3. Declared bounds and their inversion}

\begin{center}
{\renewcommand{\baselinestretch}{1}\small Table 1: Notation. The four \(\tau\)
are the declared errors; the \(\varepsilon\) are what Section~3 inverts them
into.\par}
\vspace{1mm}
{\renewcommand{\baselinestretch}{1}\scriptsize\setlength{\tabcolsep}{4.5pt}
\begin{tabular}{| c | l | c | l |}
\hline
\textbf{Symbol} & \textbf{Meaning} & \textbf{Symbol} & \textbf{Meaning} \\
\hline
\(\taupx\) & Center shift, per axis, px. &
\(\varepsilon_\mu\) & Position, from \(\taupx\) (\S3.2). \\
\hline
\(\taush\) & Footprint deformation, per mech., px. &
\(\theta\) & Rotation angle, from \(\taush\) (\S3.3). \\
\hline
\(\tauC\) & Color change, on a \(0\)--\(1\) scale. &
\(\varepsilon_\omega\) & Tangent-space \(\omega\), \(=\theta/2\sqrt3\) (\S3.3, \S4.3). \\
\hline
\(\taual\) & Opacity change, on a \(0\)--\(1\) scale. &
\(\varepsilon_s\) & Log-scale, from \(\taush\) (\S3.3). \\
\hline
\(\mathcal D(z_{\min},f_{\max},\phi)\) & Camera domain: near plane, focal, angle. &
\(\varepsilon_k\) & SH coefficient \(k\), from \(\tauC\) (\S3.4). \\
\hline
\(\delta_{x,y,z}\) & Center displacement, camera axes. &
\(\varepsilon_\alpha\) & Opacity, \(=\taual\) (\S3.4). \\
\hline
\(\delta_{u,v}\) & Its projection, px. &
\(\varepsilon_{i,j}\) & Half-step actually used; \(\Delta=2\varepsilon\). \\
\hline
\(d_H\) & Footprint shift, Hausdorff, px. & & \\
\hline
\end{tabular}}
\end{center}
\addtocounter{table}{1}

\noindent The four tolerances do not reach equally far into the rendered
image. \(\taupx\) and \(\taush\) bound how far drawn elements move, in
pixels, at every camera in \(\mathcal D\); neither bounds a pixel value.
\(\tauC\) bounds each primitive's radiance in every direction, and reaches the
pixel itself only with the geometry and the alphas held fixed (\S3.4).
\(\taual\) is a parameter bound with no image-space consequence.

\SubSection{3.1 The contract}
Write \(x_{i,j}\) for channel \(j\) of primitive \(i\) and \(P_{i,j}\) for
the decoder's prediction of it. Channels are coded as stored for centers and
SH, in the log domain for scales, as the tangent-space residual of \S4.3 for
rotation, and as \(\alpha\) after the sigmoid for opacity, which is why
\(\varepsilon_\alpha=\taual\); \S5.3 measures each error in the same domain. Sections~3.2--3.4 turn the four declared
targets into per-primitive, per-channel half-steps \(\varepsilon_{i,j}\).
The encoder transmits one integer per channel,
\begin{equation}
k_{i,j}=\Big\lceil\frac{r_{i,j}}{\Delta_{i,j}}\Big\rfloor,\qquad
r_{i,j}=x_{i,j}-P_{i,j},\qquad \Delta_{i,j}=2\varepsilon_{i,j},
\label{eq:quant}
\end{equation}
and the decoder reconstructs \(\hat x_{i,j}=P_{i,j}+k_{i,j}\Delta_{i,j}\).

\begin{proposition}[predictor independence]\label{prop:contract}
\(|\hat x_{i,j}-x_{i,j}|\le\varepsilon_{i,j}\) for all \(i,j\) and any
prediction \(P_{i,j}\).
\end{proposition}

\begin{proof}
Rounding gives \(|\hat x-x|=|k\Delta-r|\le\Delta/2=\varepsilon\), for any
\(P\).
\end{proof}

The \(\varepsilon\) are derived against a \emph{declared camera domain}
\(\mathcal D\): the cameras for which the guarantee holds, fixed by a near plane
\(z_{\min}\), a largest focal length \(f_{\max}\) and a half-angle \(\phi\) ---
a constraint on imaging parameters, not on where the camera center lies. Every
bound below is a worst case over \(\mathcal D\), and covers cameras never
supplied to the encoder.

Per attribute the per-primitive bounds are quantized into a transmitted
eight-entry table, log-spaced above a floor at the minimum and always rounded
down, so that no primitive receives a step larger than its own bound.

\SubSection{3.2 Position (\(\taupx\))}
A primitive sits at camera-space \((x,y,z)\), \(z>0\), under focal length \(f\),
and projects to \(u=fx/z\).

\begin{proposition}[position]\label{prop:position}
If every world coordinate of the center is held to at most
\(\varepsilon_\mu=\taupx z_{\min}/[\sqrt3\,(f_{\max}\sec\phi+\taupx)]\),
then \(|\delta_u|\le\taupx\) and \(|\delta_v|\le\taupx\) at every camera in
\(\mathcal D\), for every primitive whose center projects inside its frustum.
\end{proposition}

\begin{proof}
Take \(f_{\max}\) and \(\phi\) over both image axes; WLOG we bound
\(|\delta_u|\). A bound of \(\varepsilon_\mu\)
on each world coordinate puts the displacement in a cube, so
\(\lVert\delta\rVert_2\le\rho\) with
\(\rho=\sqrt3\,\varepsilon_\mu=\taupx z_{\min}/(f_{\max}\sec\phi+\taupx)\),
so \(\rho<z_{\min}\) because \(f_{\max}\sec\phi>0\); rotation
into camera axes preserves the norm, so
\(\lVert(\delta_x,\delta_z)\rVert_2\le\rho\) and \(|\delta_z|\le\rho\) there as
well. With \(u=fx/z\) the projected displacement is exact,
\[
|\delta_u|\;=\;\tfrac{f\,|z\,\delta_x-x\,\delta_z|}{z\,|z+\delta_z|}
 \mathrel{\overset{\textup{(i)}}{\le}}
   \tfrac{f\sqrt{z^{2}+x^{2}}\,\lVert(\delta_x,\delta_z)\rVert_2}{z\,|z+\delta_z|}
 \mathrel{\overset{\textup{(ii)}}{\le}}\tfrac{\rho\,f\sqrt{z^{2}+x^{2}}}{z\,(z-\rho)}
 \mathrel{\overset{\textup{(iii)}}{\le}}\tfrac{\rho\,f\sec\phi}{z-\rho}
 \mathrel{\overset{\textup{(iv)}}{\le}}\tfrac{\rho\,f_{\max}\sec\phi}{z_{\min}-\rho}
 \le\taupx .
\]
\[
\rho f_{\max}\sec\phi\le\taupx(z_{\min}-\rho)
\Leftrightarrow\sqrt3\,\varepsilon_\mu(f_{\max}\sec\phi+\taupx)\le\taupx z_{\min}
\Leftrightarrow\varepsilon_\mu\le
   \tfrac{\taupx z_{\min}}{\sqrt3(f_{\max}\sec\phi+\taupx)}.
\]
Here \textup{(i)} is Cauchy--Schwarz on \((z,-x)\!\cdot\!(\delta_x,\delta_z)\);
\textup{(ii)} uses \(\lVert(\delta_x,\delta_z)\rVert_2\le\rho\) and
\(|z+\delta_z|\ge z-\rho\); \textup{(iii)} is \(|x|\le z\tan\phi\), so
\(\sqrt{z^{2}+x^{2}}\le z\sec\phi\); and \textup{(iv)} takes the extremes of a
quantity increasing in \(f\) and decreasing in \(z\).
\end{proof}

\noindent One number therefore serves every primitive. A per-primitive bound
minimized over the reference poses alone would be tighter, but it is violated
on held-out cameras in all four scenes we tried. Since step \textup{(iii)}
assumes \(|x|\le z\tan\phi\), the center guarantee applies to primitives whose
center projects inside the frustum; the footprint bound of \S3.3 has no such
restriction because the rasterizer clamps \(x/z\), though \S5.3 measures both
on the in-frustum set.

\SubSection{3.3 Footprint: rotation and scale (\(\taush\))}
A 3DGS rasterizer draws an EWA footprint, the \(1\sigma\) ellipse of the
projected covariance \(\Sigma'=J\Sigma J^{\!\top}\), where \(J\) is the
projection Jacobian evaluated
at the primitive center and the off-axis ratios are clamped to \(m\tan\phi\)
before \(J\) is formed (\(m=1.3\) in the INRIA rasterizer and gsplat). Write
\(\Sigma=MM^{\!\top}\), \(M=R\,\mathrm{diag}(s)\), \(s_{\max}=\max_j s_j\), and let \(d_H\) be the
Hausdorff distance between the reference footprint and the decoded one.

\begin{proposition}[footprint]\label{prop:footprint}
Let \(\ell_x,\ell_y\) be the clamped off-axis limits,
\(\sec\phi_J=\sqrt{1+\ell_x^{2}+\ell_y^{2}}\), and
\(r_i=f_{\max}\sec\phi_J\,s_{\max,i}/z_{\min}\). If primitive \(i\) is rotated
by a geodesic angle of at most \(\theta_i\), or its log-scales change per axis
by at most \(\varepsilon_{s,i}\), where \(\theta_i=\taush/r_i\) and
\(\varepsilon_{s,i}=\log(1+\taush/r_i)\), then its footprint moves by at most
\(\taush\) in Hausdorff distance, at a
fixed center and at every camera in \(\mathcal D\).
\end{proposition}

\begin{proof}
The footprint is convex with support function
\(h(n)=\lVert M^{\!\top}J^{\!\top}n\rVert\), so
\(d_H=\max_{\lVert n\rVert=1}|h'(n)-h(n)|\), and
\[
d_H
\mathrel{\overset{\textup{(i)}}{\le}}\lVert(M'-M)^{\!\top}J^{\!\top}n\rVert
\mathrel{\overset{\textup{(ii)}}{\le}}\lVert J\rVert_2\lVert M'-M\rVert_2
\mathrel{\overset{\textup{(iii)}}{\le}}
   \tfrac{f_{\max}\sec\phi_J}{z_{\min}}\lVert M'-M\rVert_2
= r_i\,\tfrac{\lVert M'-M\rVert_2}{s_{\max,i}} .
\]
Here \textup{(i)} is the reverse triangle inequality and \textup{(ii)}
submultiplicativity, both exact; \textup{(iii)} replaces \(\lVert J\rVert_2\)
by its supremum over \(\mathcal D\). The singular values of
\(J=\frac{f}{z}\bigl[\begin{smallmatrix}1&0&-x/z\\0&1&-y/z\end{smallmatrix}
\bigr]\) are \(f/z\) and \((f/z)\sqrt{1+(x^{2}+y^{2})/z^{2}}\), so
\(\lVert J\rVert_2\) carries an off-axis factor, amplification along both image
axes; the clamp bounds \(|x|/z\le\ell_x\) and \(|y|/z\le\ell_y\), and
\(z\ge z_{\min}\), so \(\sup_{\mathcal D}\lVert J\rVert_2=f_{\max}\sec\phi_J/
z_{\min}\) --- finite for every off-axis position because of the clamp.

Each mechanism now bounds \(\lVert M'-M\rVert_2/s_{\max,i}\). A rotation by
\(\theta_i\) has \(\lVert M'-M\rVert_2\le\lVert R'-R\rVert_2\,s_{\max,i}
=2s_{\max,i}\sin(\theta_i/2)\le s_{\max,i}\theta_i\), since \(R'R^{\!\top}\) is
normal with eigenvalues \(1,e^{\pm i\theta_i}\); a log-scale change gives
exactly \(\max_j s_j|e^{\lambda_j}-1|\le s_{\max,i}(e^{\varepsilon_{s,i}}-1)\).
Hence
\[
\theta_i\,r_i\le\taush\Leftrightarrow\theta_i\le\frac{\taush}{r_i},
\qquad
(e^{\varepsilon_{s,i}}-1)\,r_i\le\taush\Leftrightarrow
\varepsilon_{s,i}\le\log\Bigl(1+\frac{\taush}{r_i}\Bigr). \qedhere
\]
\end{proof}

\noindent These two are what the quantizer receives, and unlike
\(\varepsilon_\mu\) they stay primitive-dependent through \(s_{\max,i}\). \(\varepsilon_{s,i}\) is already a
half-step; \(\theta_i\) is an angle, and \S4.3 codes it as
\(\varepsilon_{\omega,i}=\theta_i/2\sqrt3\) on the tangent-space residual.

Rotation and scale together therefore displace the footprint by at most
\(2\taush\) at a fixed center. When the center also moves, its displacement
contributes an additional term and changes the projection Jacobian. With the
center included, \(h\) gains the term \(\langle n,\text{projected center}\rangle\),
whose change is \(\lVert(\delta_u,\delta_v)\rVert_2\le\sqrt2\,\taupx\), since
Proposition~\ref{prop:position} bounds each axis; writing \(A=JM\), the
decomposition \(A'-A=J'(M'-M)+(J'-J)M\) gives \(\lVert A'-A
\rVert_2\le\lVert J'\rVert_2\lVert M'-M\rVert_2+\lVert J'-J\rVert_2 s_{\max}\),
the first term at most \(2\taush\) by the two mechanisms above. Together,
\begin{equation}
d_H\;\le\;\underbrace{\sqrt2\,\taupx}_{\text{center}}+\underbrace{2\taush}_{\text{shape}}
 +\underbrace{s_{\max}\lVert J'-J\rVert_2}_{\delta_J},
\qquad
\delta_J\le\frac{\Lambda}{\sec\phi}\,\frac{s_{\max}}{z_{\min}}\,\taupx,
\label{eq:composite}
\end{equation}
where \(\Lambda=\sqrt{2+(1+2\ell_x)^{2}+(1+2\ell_y)^{2}}\).
The measured displacement in \S5.3 stays well within \(\sqrt2\,\taupx+2\taush\), so
\(\delta_J\) is negligible in practice. Finally, \(h\) is homogeneous in
\(M\), so the shape displacement of a \(c\sigma\) contour scales exactly by
\(c\).

\SubSection{3.4 Color and opacity (\(\tauC\), \(\taual\))}
For viewing direction \(d\), the SH radiance is \(\textstyle\sum_k c_kY_k(d)\),
so an error in coefficient \(k\) moves the color by a direction-dependent
amount and \(\tauC\) must be inverted over the whole sphere. Opacity is direct,
\(\varepsilon_\alpha=\taual\).

\begin{proposition}[color]\label{prop:color}
Let \(g(d)=\sum_k\varepsilon_k|Y_k(d)|\) be the largest radiance error that
per-coefficient bounds \(\varepsilon_k\) permit at \(d\). If
\(\max_{S^2}g\le\tauC\), then every primitive's radiance stays within
\(\tauC\) in every viewing direction.
\end{proposition}

\begin{proof}
\(|\hat c(d)-c(d)|=\bigl|\sum_k(\hat c_k-c_k)Y_k(d)\bigr|
\le\sum_k\varepsilon_k|Y_k(d)|=g(d)\).
\end{proof}

\noindent Decoupling the coefficients --- bounding each term by its own
maximum, \(\varepsilon_k=\tauC/(16\sup|Y_k|)\) --- is loose, because the peaks
of different \(Y_k\) fall in different directions: the binding quantity is the
maximum of the sum, not the sum of individual maxima. Solving
\(\max_{S^2}g\le\tauC\) as it stands admits \(2.3\times\) larger \(\varepsilon_k\), \(1.2\) bits per
coefficient.

Subject to that constraint, we seek large quantization bounds by maximizing
\(\sum_k n_k\log\varepsilon_k\), with \(n_k=3\) channels per coefficient.
Rather than solving this exactly, the encoder alternates between a
worst-direction search and water-filling; correctness depends only on the
feasibility of the returned bounds, not on the optimality of this search.

To extend a check at sampled directions to the whole sphere, we reserve a
Lipschitz safety margin. \(g\) is Lipschitz with
\(L=\sum_k\varepsilon_k\mathrm{Lip}(Y_k)\), from Bernstein's inequality
\(\mathrm{Lip}(Y_\ell)\le\sqrt{\ell(\ell+1)}\sup|Y_\ell|\) and the addition
theorem \(\sup|Y_\ell^m|\le\sqrt{(2\ell+1)/4\pi}\); a net of covering radius
\(\eta\) then satisfies \(\max_{S^2}g\le\max_{\mathrm{net}}g+L\eta\), and we
allocate against this corrected maximum. We choose a Fibonacci net of
\(M=2^{20}\) points and take \(\eta\le1.05\cdot2\sqrt{\pi/M}\). This \(\eta\)
is asserted and confirmed by sampling, not computed exactly: it is the one
numerical assumption in the color bound, where the rest of \S3 is
analytic.\footnote{Numerical search estimates the covering radius as
\(0.0024\) rad, versus the declared \(0.0036\) rad. The Lipschitz correction
\(L\eta\) is \(2.4\%\) of \(\tauC\), with a \(0.7\%\) rate overhead.}

\begin{proposition}[composited color]\label{prop:blend}
Fix the geometry, so a pixel composites primitives \(1,\dots,n\) in depth order
with unchanged alphas \(a_i\in[0,1]\). If every primitive satisfies the color
contract then \(|\hat C-C|\le\tauC\) for that pixel, in every viewing direction.
\end{proposition}

\begin{proof}
\(C=\sum_i w_ic_i\) with \(w_i=a_i\prod_{j<i}(1-a_j)\ge0\) and
\(\sum_{i\le n}w_i=1-\prod_i(1-a_i)\le1\), so
\(|\hat C-C|\le\tauC\sum_i w_i\le\tauC\); the final clamp is 1-Lipschitz.
\end{proof}

\noindent Alpha compositing is a sub-convex combination, so \(\tauC=0.1\)
declares that no pixel moves by more than a tenth of the color range, at any
depth complexity. Geometry admits no comparable bound on the rendered image: a
high-contrast edge crossing a pixel changes it by \(O(1)\) under an arbitrarily
small displacement. That is why \(\taupx\) and \(\taush\) are declared in
pixels of displacement rather than in color.

\Section{4. The codec}

\SubSection{4.1 Pruning}

Before coding, we retain by default the top \(60\%\) of primitives, ranked by
\(s_i=\alpha_i\textstyle\sum_{v\in\mathcal V} r_{i,v}^{2}\), an image-free
variant of LightGaussian's global significance~\cite{fanlightgaussian}: it
uses camera poses and intrinsics but no pixels, and no fine-tuning follows.
Pruning lies outside the guarantee, so every declared tolerance is relative to
the pruned model; at \(60\%\) keep it costs \(0.082\,\mathrm{dB}\) on
Mip-NeRF~360.

\SubSection{4.2 Hierarchy and ordering}
We Morton-order the retained primitives, then split them into five hierarchy
levels. Four grids with resolutions \(32\), \(128\), \(512\), and \(2048\) are
processed in order; each occupied cell contributes its highest-scoring
primitive not already selected, and the remainder form the fifth level, so
early levels carry both spatial coverage and high-importance primitives.

Each primitive in level \(\ell>0\) takes its nearest neighbor in earlier levels
as its parent, and we stably sort the level by parent index so that each
parent's children are contiguous. The predictor then addresses neighboring
children by array offset instead of explicit indices; the decoder rebuilds the
ordering from child counts, escape-coded into the first payload blobs.

\SubSection{4.3 Prediction and quantization}
Level \(0\) is predicted from the transmitted global mean; later levels
predict each primitive from its parent \(p(i)\) and that parent's \(S-1\)
array successors (\(S=8\)): \(P_i=\sum_{j\in\mathcal N(i)} w_{ij}\hat
X^{\mathrm b}_j/\sum_j w_{ij}\) with
\(w_{ij}=(\|\hat\mu^{\mathrm b}_{p(i)}-\hat\mu^{\mathrm b}_{j}\|+d_0)^{-2}\)
and a per-level floor \(d_0\) in the header.

The predictor reads only the base-precision reconstruction
\(\hat X^{\mathrm b}\), the top \(B-5\) planes, an anchor \(32\times\) coarser
than the final value, so a decoder that stops there computes the encoder's
prediction exactly.

Rotation is coded in the tangent space of its parent,
\(\omega=\log(q_{p(i)}^{*}q_i)\). Proposition~\ref{prop:contract} bounds each
component of \(\omega\) by \(\varepsilon_\omega\), hence the geodesic angle
between \(q\) and \(\hat q\) by \(2\sqrt3\,\varepsilon_\omega\).
Quantization is then \eqref{eq:quant}:
prediction shapes the entropy of \(k\) but cannot move the error past
\(\varepsilon\).

\SubSection{4.4 Bit-plane coding}

Integer residuals are encoded MSB-first over \(B\) planes, \(B\) the bit
length of the largest \(|k|\) in the model: each plane emits one bit per
coefficient and sends the sign when a coefficient first becomes significant.
Bits within a plane are ordered by attribute --- position, log-scale,
rotation, opacity, degree-0 SH, degree-1--3 SH --- and, within each attribute,
significance bits precede refinement bits, then the whole plane goes to the
entropy coder as one stream. Grouping bits of similar statistics improves
adaptation without separate contexts, and coding the plane as one stream
leaves no per-plane length metadata to send.

\SubSection{4.5 Progression order and the container}
We group \(G\) consecutive bit-planes into one blob coded with a batched ANS
coder~\cite{duda2013ans}, giving \(N_U=\lceil B/G\rceil\) blobs per level.
Ordering is three-stage: (1) levels \(\ell\le c\) complete, (2) all levels to
base precision, (3) refinement blobs across levels. With \(c=1\), Stage 2 adds
primitives, Stage 3 precision.

The default \(G\ge B\) puts a whole level in one blob, so a prefix lands only
on a level boundary, progressive in primitive count alone. Opening the
precision axis fully costs \(18\%\) more rate, since one blob per plane
discards cross-plane probability adaptation; \(G=3\) gives usable rungs at
\(9\%\). All measurements use the default except Section~5.4.

Header, ordering table, and blob index precede the payload: the header stores
the four declared tolerances, the camera domain \(z_{\min}, f_{\max}, \phi\),
and the quantization-bound tables, so a decoder or verifier recovers the
declared contract from any prefix alone.

\SubSection{4.6 The truncation ladder}
\begin{proposition}[the ladder]\label{prop:ladder}
In a prefix where every delivered level has reached base precision, let level
\(\ell\) hold \(p_\ell\) of \(B\) planes and write \(m=B-p_\ell\). Then
\(|\hat x_{i,j}-x_{i,j}|\le(2^{\,m+1}-1)\varepsilon_{i,j}\) for every primitive
of that level, and the constant cannot be improved; levels not yet delivered
are absent, not wrong.
\end{proposition}

\begin{proof}
At base precision the anchor, hence \(P\) and \(k\), is the encoder's (\S4.3),
so \(|P+k\Delta-x|\le\varepsilon\) (Proposition~\ref{prop:contract}). From the
top \(p_\ell\) bits the decoder takes the midpoint of the interval they leave,
or \(0\) when they are all zero and \(|k|\le2^{m}-1\); either way
\(|\hat k-k|\le2^{m}-1\) and \(|\hat x-x|\le(2^{m}-1)2\varepsilon+\varepsilon\).
It is tight at \(r/\Delta\uparrow2^{m}-\tfrac12\), the largest magnitude those
planes code as zero: \(\hat k=0\) and \(|\hat x-x|\to(2^{\,m+1}-1)\varepsilon\).
\end{proof}

\noindent Measurements on three scenes reach it to four decimals at every rung.
Below base precision the anchors differ and no bound is claimed, so every level
reaches base precision before refinement.

\Section{5. Experiments}

\SubSection{5.1 Protocol}
\begin{figure}[tb]
\centering
\begin{minipage}[t]{0.475\linewidth}\centering
\vspace{0pt}
\includegraphics[width=0.75\linewidth]{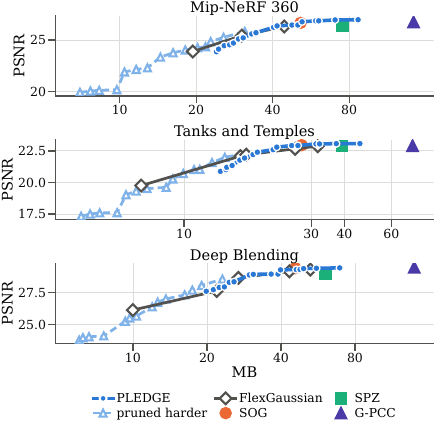}\\[1pt]
{\small (a) rate--distortion}
\end{minipage}\hfill
\begin{minipage}[t]{0.505\linewidth}\centering
\vspace{0pt}
\includegraphics[width=0.8\linewidth]{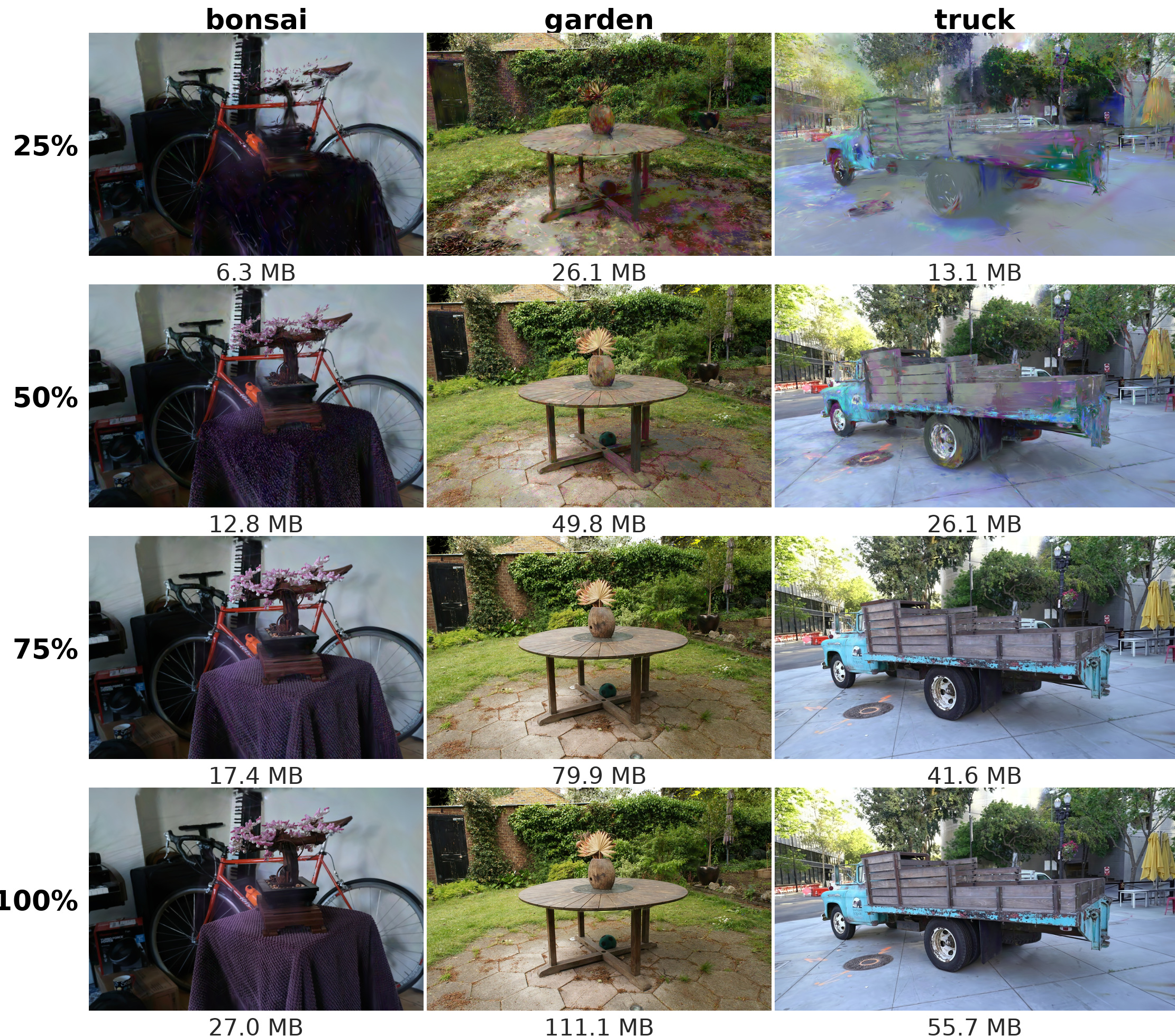}\\[1pt]
{\small (b) a partial download}
\end{minipage}
\caption{\label{fig:rd}\textbf{(a)} Rate--distortion. Solid: \(60\%\) keep; dashed: a front over keep
fractions \(45\%\) down to \(12\%\), whose bounds are relative to smaller
models. \textbf{(b)} Progressive decoding at \(25/50/75/100\%\) of the byte budget;
every valid prefix satisfies Proposition~\ref{prop:ladder}. Test-split PSNR
for the full sweep is reported in \S5.4.\label{fig:progvis}}
\end{figure}

\begin{table}[tb]
\centering
{\renewcommand{\baselinestretch}{1}\scriptsize\setlength{\tabcolsep}{2.4pt}
\begin{minipage}[t]{0.475\linewidth}\centering
\vspace{0pt}
\input{tables/guarantee}\\[2pt]
(a) the guarantee
\end{minipage}\hfill
\begin{minipage}[t]{0.505\linewidth}\centering
\vspace{0pt}
\input{tables/main}\\[2pt]
(b) thirteen-scene mean
\end{minipage}}
\caption{\label{tab:guarantee}\textbf{(a)} Worst error over thirteen scenes at nine
\((\taupx,\taush,\tauC)\) settings, each entry divided by its own bound.
Upper block by attribute; lower block by declared \(\tau\) in image space,
against \(\taupx\), \(2\taush\) and \(\sqrt2\taupx{+}2\taush\) --- the bound
\eqref{eq:composite} with the loose \(\delta_J\) term dropped. Footprint rows are Hausdorff distances
between the rasterizer's \(\Sigma'\) for the reference and the decoded model,
computed independently of the bounds in \S3.3.
\textbf{(b)} Thirteen-scene mean at three \ours\ operating points,
\(\tauC=0.8\), \(0.2\) and \(0.1\) for (A), (B) and (C), all at
\((\taupx,\taush,\taual)=(8,16,0.005)\); (C) matches SOG in size.\label{tab:main}}
\end{table}

\begin{figure}[tb]
\centering
\includegraphics[width=5in]{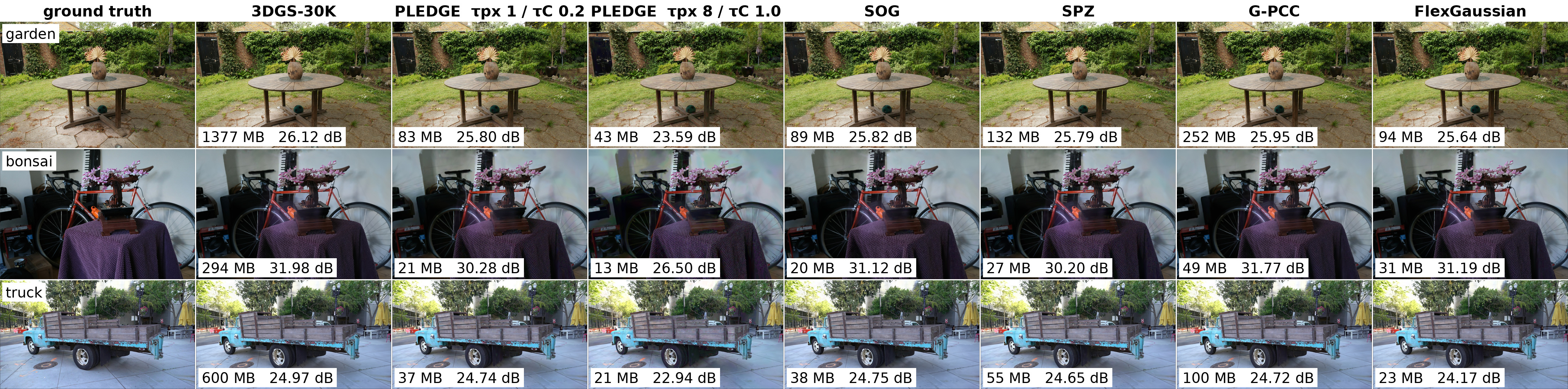}
\vspace{-2em}
\caption{\label{fig:qual}Representative held-out views, each codec at its own operating point.}
\end{figure}

We use the official INRIA 3DGS-30K models for nine Mip-NeRF~360
scenes~\cite{barron2022mipnerf360}, two Tanks and
Temples~\cite{knapitsch2017tanks} and two Deep
Blending~\cite{hedman2018deepblending}, under the survey
protocol~\cite{bagdasarian3dgszip}. All methods, FlexGaussian included, are
re-rendered with gsplat~\cite{ye2025gsplat} under one metric implementation.
Unless stated otherwise, \(\taual=0.005\) and \(z_{\min}=0.2\). Nothing is
trained: over the nine Mip-NeRF~360 scenes, bound inversion and the residual
loop together take \(0.4\)--\(1.2\) seconds per scene on one A100; model
loading and entropy coding take the rest of the wall clock.

\SubSection{5.2 Rate--distortion}
Table~\ref{tab:main}(b) and Figure~\ref{fig:rd}(a) compare rate--distortion
performance, and Figure~\ref{fig:qual} shows representative reconstructions.
Across the thirteen scenes, \ours\ achieves compression ratios of
\(15\)--\(30\times\). We sweep the four tolerances and report Pareto-optimal
configurations.

Over four scenes, relaxing \(\taush\) reduces file size by \(21\%\) with less
than \(0.05\,\mathrm{dB}\) PSNR loss, while relaxing \(\taupx\) reduces it by
\(13\%\) and \(\taual\) by only \(5\%\). Bit allocation moves with the
tolerances: degree 1--3 SH takes \(48\%\) of the file at \(\tauC=0.05\) but
\(11\%\) at \(\tauC=0.4\), where position grows from \(14\%\) to \(26\%\).

At matched quality, \ours\ produces smaller files at \(11\) of \(19\)
comparable training-free baseline operating points. At the lowest rates, the
fixed \(60\%\) keep fraction becomes the bottleneck:
\ours\ requires \(1.24\times\) FlexGaussian's size at its loosest point.
Reducing the keep fraction to \(12\%\) removes that deficit on Mip-NeRF~360
and extends our curve \(3.4\times\) further toward low rates, at the cost of
\(1.2\,\mathrm{dB}\) lower maximum PSNR.

\SubSection{5.3 Verifying the declared bounds}

We evaluate nine settings on all thirteen scenes. Six take
\(\taush=\taupx\): three at \(\taupx=1\) with \(\tauC=0.1\), \(0.2\) and
\(0.5\), then \((\taupx,\tauC)=(2,0.2)\), \((4,0.5)\) and \((8,1.0)\). The
other three are the operating points of Table~\ref{tab:guarantee}(b), at
\(\taupx=8\) and \(\taush=16\) with \(\tauC=0.1\), \(0.2\) and \(0.8\).
Throughout, \(\taual=0.005\). For each file, we
recover the declared tolerances from its header and compare them with the
decoded attribute errors. All
\(585\) parameter-space checks satisfy their bounds within numerical precision
(Table~\ref{tab:guarantee}(a), top row).

We independently validate the image-space bounds of
Propositions~\ref{prop:position} and~\ref{prop:footprint} by reconstructing the
rasterizer's projected covariances, \(\Sigma'=J\Sigma J^{\!\top}\), for both
the reference and decoded primitives, and measuring the Hausdorff distance
between their footprints through their support functions. Center displacement,
footprint deformation at fixed center, and their combination reach at most
\(0.77\,\taupx\), \(0.35\cdot2\taush\) and \(0.41\,(\sqrt2\taupx+2\taush)\),
respectively (Table~\ref{tab:guarantee}(a), lower block). At \(\taupx=1\) and
\(\tauC=0.1\), measured displacements range from \(0.072\) pixels on
\emph{garden} to \(0.582\) on \emph{train}.

The off-axis factor \(\sec\phi_J\) ranges from \(1.13\) (\emph{bonsai}) to
\(1.61\) (\emph{train}); omitting it violates the footprint bound on off-axis
primitives, and including it costs \(1.7\%\) of the rate on average.

\SubSection{5.4 Progressive decoding}

Figure~\ref{fig:progvis}(b) decodes byte-budgeted prefixes, the client given
nothing but the prefix. Over the nine Mip-NeRF~360 scenes test-split PSNR
rises monotonically on every scene at the ten budgets we sample, from
\(14.7\,\mathrm{dB}\) at \(10\%\) to \(26.9\) at the full stream; the last
primitive arrives at \(40\%\) on average, and at \(50\%\) of the file the mean
gap to the full decode is \(2.0\,\mathrm{dB}\), the worst \(3.9\). The sweep
opens the precision axis fully, giving the finest rungs the format allows,
while the figure uses \(G=3\).

\Section{6. Limitations and conclusion}
The guarantee is stated relative to the pruned model and within the declared
camera domain, covers the rasterized footprint rather than the exact
perspective silhouette, and bounds composited color only at fixed geometry and opacity.
Within this scope, \ours\
provides interpretable, per-primitive bounds across a camera domain and
throughout valid progressive decoding. Because the bounds are independent of
prediction and entropy coding, improved residual coders can be incorporated
without changing their derivation.

\Section{References}
\bibliographystyle{IEEEbib}
\bibliography{refs}

\end{document}

%% file: figs/framework.tex
\begin{tikzpicture}[
  x=1in, y=1in, font=\scriptsize,
  box/.style   = {draw=black!55, rounded corners=1.4pt, inner xsep=2pt,
                  inner ysep=2.5pt, align=center, anchor=west,
                  minimum height=0.36in},
  key/.style   = {box, draw=ourblue, line width=0.8pt},
  fl/.style    = {-{Latex[length=3.2pt,width=2.6pt]}, draw=black!55, line width=0.5pt},
  kfl/.style   = {-{Latex[length=3.4pt,width=2.8pt]}, draw=ourblue, line width=0.9pt},
  lbl/.style   = {font=\tiny, text=black!60, inner sep=0.8pt},
]
\node[key, text width=1.12in] (dec)    at (0.000,0)
  {\textbf{declared}\\[0.5pt]$\taupx\,\taush\,\tauC\,\taual$\\[0.5pt]
   domain $z_{\min},f_{\max},\phi$};
\node[key, text width=0.52in] (eps)    at (1.436,0) {per-prim.\\$\varepsilon_{i,j}$};
\node[key, text width=1.05in] (quant)  at (2.272,0)
  {\textbf{quantize}\\$k=\lceil r/\Delta\rfloor$\\[0.5pt]
   \raisebox{0pt}{\tiny $r$: prune, level, predict}};
\node[box, text width=0.56in] (planes) at (3.548,0) {bit-planes\\$+$ ANS};
\node[box, text width=0.54in] (file)   at (4.334,0) {\texttt{.gsp}\\header first};
\node[key, text width=0.84in] (prefix) at (5.100,0)
  {\textbf{any valid prefix}\\$(2^{\,B-p+1}\!\!-\!1)\varepsilon$};

\draw[kfl] (dec.east)    -- node[lbl,above]{invert}                (eps.west);
\draw[kfl] (eps.east)    -- node[lbl,above]{$\Delta{=}2\varepsilon$} (quant.west);
\draw[fl]  (quant.east)  -- (planes.west);
\draw[fl]  (planes.east) -- (file.west);
\draw[kfl] (file.east)   -- (prefix.west);
\end{tikzpicture}

%% file: tables/guarantee.tex
\begin{tabular}{@{}lrrrrr@{}}
\toprule
& $\mu$ & $s$ & $q$ & $\alpha$ & $C$ \\
\midrule
fraction of bound & 1.000 & 1.000 & 0.996 & 1.000 & 0.906 \\
\bottomrule
\end{tabular}
\\[5pt]
\begin{tabular}{@{}lrrrrr@{}}
\toprule
declared $\tau$ & 1 & 2 & 4 & 8 & 16 \\
\midrule
center & 0.672 & 0.765 & 0.626 & 0.681 & --- \\
footprint & 0.291 & 0.349 & 0.321 & 0.341 & 0.271 \\
\quad +center & 0.397 & 0.366 & 0.402 & 0.357 & 0.269 \\
\bottomrule
\end{tabular}

%% file: tables/main.tex
\begin{tabular}{@{}lrrrrr@{}}
\toprule
method & MB & ratio & PSNR & SSIM & LPIPS \\
\midrule
3DGS-30K (input) & 722.18 & 1.0$\times$ & 26.83 & 0.8252 & 0.2591 \\
\midrule
\ours\ (A) & 24.50 & 29.5$\times$ & 24.49 & 0.7694 & 0.3269 \\
\ours\ (B) & 37.00 & 19.5$\times$ & 26.21 & 0.8073 & 0.2852 \\
\ours\ (C) & 46.56 & 15.5$\times$ & 26.56 & 0.8153 & 0.2730 \\
\midrule
FlexGaussian & 37.55 & 19.2$\times$ & 26.04 & 0.7964 & 0.2910 \\
SPZ & 67.41 & 10.7$\times$ & 26.25 & 0.8126 & 0.2796 \\
SOG & 46.94 & 15.4$\times$ & 26.51 & 0.8156 & 0.2763 \\
G-PCC & 131.83 & 5.5$\times$ & 26.57 & 0.8134 & 0.2698 \\
\bottomrule
\end{tabular}

%% file: refs.bib
@article{bagdasarian3dgszip,
  author  = {Bagdasarian, Milena T. and Knoll, Paul and Li, Yi-Hsin and Barthel, Florian and Hilsmann, Anna and Eisert, Peter and Morgenstern, Wieland},
  title   = {{3DGS.zip}: A {Survey} on {3D} {Gaussian} {Splatting} {Compression} {Methods}},
  journal = {Computer Graphics Forum},
  pages   = {e70078},
  year    = {2025},
  doi     = {https://doi.org/10.1111/cgf.70078},
  url     = {https://onlinelibrary.wiley.com/doi/abs/10.1111/cgf.70078},
  note    = {\url{https://w-m.github.io/3dgs-compression-survey/}},
}

@inproceedings{barron2022mipnerf360,
  title     = {{Mip-NeRF 360: Unbounded Anti-Aliased Neural Radiance Fields}},
  author    = {Barron, Jonathan T. and Mildenhall, Ben and Verbin, Dor and Srinivasan, Pratul P. and Hedman, Peter},
  booktitle = {CVPR},
  pages     = {5460--5469},
  year      = {2022},
  doi       = {10.1109/CVPR52688.2022.00539}
}

@inproceedings{chenhac,
  author    = {Chen, Yihang and Wu, Qianyi and Lin, Weiyao and Harandi, Mehrtash and Cai, Jianfei},
  title     = {{HAC}: Hash-Grid Assisted Context for {3D} {G}aussian Splatting Compression},
  booktitle = {ECCV},
  series    = {Lecture Notes in Computer Science},
  pages     = {422--438},
  publisher = {Springer},
  year      = {2024},
  doi       = {10.1007/978-3-031-72667-5_24}
}

@inproceedings{di2016sz,
  author    = {Sheng Di and Franck Cappello},
  title     = {Fast Error-Bounded Lossy {HPC} Data Compression with {SZ}},
  booktitle = {IPDPS},
  pages     = {730--739},
  year      = {2016},
  doi       = {10.1109/IPDPS.2016.11}
}

@misc{duda2013ans,
  author       = {Jarek Duda},
  title        = {Asymmetric numeral systems: entropy coding combining speed of {H}uffman coding with compression rate of arithmetic coding},
  howpublished = {arXiv preprint arXiv:1311.2540 [cs.IT]},
  year         = {2013},
  eprint       = {1311.2540},
  archiveprefix = {arXiv},
  primaryclass = {cs.IT},
  doi          = {10.48550/arXiv.1311.2540}
}

@inproceedings{fanlightgaussian,
  author    = {Fan, Zhiwen and Wang, Kevin and Wen, Kairun and Zhu, Zehao and Xu, Dejia and Wang, Zhangyang},
  title     = {{LightGaussian}: Unbounded {3D} {G}aussian Compression with 15x Reduction and 200+ {FPS}},
  booktitle = {NeurIPS},
  volume    = {37},
  year      = {2024}
}

@article{hedman2018deepblending,
  title   = {{Deep Blending for Free-Viewpoint Image-Based Rendering}},
  author  = {Hedman, Peter and Philip, Julien and Price, True and Frahm, Jan-Michael and Drettakis, George and Brostow, Gabriel},
  journal = {ACM Trans. Graph.},
  volume  = {37},
  number  = {6},
  pages   = {257:1--257:15},
  year    = {2018},
  doi     = {10.1145/3272127.3275084}
}

@article{kerbl3dgs,
  title   = {{3D Gaussian Splatting for Real-Time Radiance Field Rendering}},
  author  = {Kerbl, Bernhard and Kopanas, Georgios and Leimk{\"u}hler, Thomas and Drettakis, George},
  journal = {ACM Trans. Graph.},
  volume  = {42},
  number  = {4},
  pages   = {139:1--139:14},
  year    = {2023},
  doi     = {10.1145/3592433}
}

@article{knapitsch2017tanks,
  title   = {{Tanks and Temples: Benchmarking Large-Scale Scene Reconstruction}},
  author  = {Knapitsch, Arno and Park, Jaesik and Zhou, Qian-Yi and Koltun, Vladlen},
  journal = {ACM Trans. Graph.},
  volume  = {36},
  number  = {4},
  pages   = {78:1--78:13},
  year    = {2017},
  doi     = {10.1145/3072959.3073599}
}

@inproceedings{luscaffoldgs,
  author    = {Lu, Tao and Yu, Mulin and Xu, Linning and Xiangli, Yuanbo and Wang, Limin and Lin, Dahua and Dai, Bo},
  title     = {{Scaffold-GS}: Structured {3D} {G}aussians for View-Adaptive Rendering},
  booktitle = {CVPR},
  pages     = {20654--20664},
  year      = {2024},
  doi       = {10.1109/CVPR52733.2024.01952}
}

@inproceedings{morgensternsog,
  author    = {Morgenstern, Wieland and Barthel, Florian and Hilsmann, Anna and Eisert, Peter},
  title     = {Compact {3D} Scene Representation via Self-Organizing {G}aussian Grids},
  booktitle = {ECCV},
  series    = {Lecture Notes in Computer Science},
  pages     = {18--34},
  publisher = {Springer},
  year      = {2024},
  doi       = {10.1007/978-3-031-73013-9_2}
}

@misc{mpegtmc13,
  author       = {{MPEG Group}},
  title        = {{mpeg-pcc-tmc13}: Geometry-Based Point Cloud Compression ({G-PCC}) Test Model},
  howpublished = {\url{https://github.com/MPEGGroup/mpeg-pcc-tmc13}},
  year         = {2023},
  note         = {Reference software for {ISO/IEC} 23090-9; builds the \texttt{tmc3} encoder/decoder}
}

@misc{nianticspz,
  author       = {{Niantic Labs}},
  title        = {{SPZ}: A File Format for Compressed {3D} {G}aussian Splats},
  howpublished = {\url{https://github.com/nianticlabs/spz}},
  year         = {2024},
  note         = {Open-source software (MIT license); no associated publication}
}

@inproceedings{niedermayrcompressed3dgs,
  author    = {Niedermayr, Simon and Stumpfegger, Josef and Westermann, R{\"u}diger},
  title     = {Compressed {3D} {G}aussian Splatting for Accelerated Novel View Synthesis},
  booktitle = {CVPR},
  pages     = {10349--10358},
  year      = {2024},
  doi       = {10.1109/CVPR52733.2024.00985}
}

@article{shapiro1993ezw,
  author  = {Jerome M. Shapiro},
  title   = {Embedded image coding using zerotrees of wavelet coefficients},
  journal = {IEEE Trans. Signal Process.},
  volume  = {41},
  number  = {12},
  pages   = {3445--3462},
  year    = {1993},
  doi     = {10.1109/78.258085}
}

@book{taubman2002jpeg2000,
  author    = {David S. Taubman and Michael W. Marcellin},
  title     = {{JPEG2000}: Image Compression Fundamentals, Standards and Practice},
  publisher = {Kluwer Academic Publishers},
  address   = {Boston, MA, USA},
  series    = {The Kluwer International Series in Engineering and Computer Science},
  volume    = {642},
  year      = {2002},
  doi       = {10.1007/978-1-4615-0799-4}
}

@inproceedings{tianflexgaussian,
  author    = {Tian, Boyuan and Gao, Qizhe and Xianyu, Siran and Cui, Xiaotong and Zhang, Minjia},
  title     = {{FlexGaussian}: Flexible and Cost-Effective Training-Free Compression for {3D} {G}aussian Splatting},
  booktitle = {ACM MM},
  pages     = {7287--7296},
  publisher = {ACM},
  year      = {2025},
  doi       = {10.1145/3746027.3754744}
}

@article{ye2025gsplat,
  title   = {{gsplat: An Open-Source Library for Gaussian Splatting}},
  author  = {Ye, Vickie and Li, Ruilong and Kerr, Justin and Turkulainen, Matias and Yi, Brent and Pan, Zhuoyang and Seiskari, Otto and Ye, Jianbo and Hu, Jeffrey and Tancik, Matthew and Kanazawa, Angjoo},
  journal = {J. Mach. Learn. Res.},
  volume  = {26},
  number  = {34},
  pages   = {1--17},
  year    = {2025},
  url     = {http://jmlr.org/papers/v26/24-1476.html}
}

@inproceedings{chen2026pcgs,
  author    = {Chen, Yihang and Li, Mengyao and Wu, Qianyi and Lin, Weiyao and Harandi, Mehrtash and Cai, Jianfei},
  title     = {{PCGS}: Progressive Compression of {3D} {G}aussian Splatting},
  booktitle = {AAAI},
  year      = {2026},
  note      = {arXiv:2503.08511}
}
